\documentclass[letterpaper, 10 pt, conference]{ieeeconf}

\IEEEoverridecommandlockouts
\usepackage{amsmath}
\usepackage{amsfonts}
\usepackage{amssymb}
\usepackage{xcolor}
\usepackage{bbold}
\usepackage{graphicx}
\usepackage{tabularx}
\usepackage{ulem}
\usepackage{mathtools}
\usepackage{verbatim}
\usepackage{lipsum}

\newtheorem{theorem}{Theorem}
\newtheorem{lemma}[theorem]{Lemma}
\newtheorem{definition}[theorem]{Definition}
\newtheorem{proposition}[theorem]{Proposition}

\newtheorem{corollary}[theorem]{Corollary}

\title{\LARGE \bf
Optimally Trading Energy for Information through Speed Trajectory Optimization of Sustainably Powered Mobile Robots
}

\author{Kavin M. Govindarajan, Dimitra Panagou, and Chris Vermillion%
\thanks{Supported by NSF Award Numbers 2223844 and DGE 2241144.}%
\thanks{This work has been submitted to the IEEE for possible publication. Copyright may be transferred without notice, after which this version may no longer be accessible.}%
\thanks{Kavin M. Govindarajan is a Graduate Research Assistant with the Department of Robotics, University of Michigan, Ann Arbor, MI 48109, USA.     {\tt\small kmgovind@umich.edu}}%
\thanks{Dimitra Panagou is an Associate Professor with the Department of Robotics, University of Michigan, Ann Arbor, MI, 48109, USA. {\tt\small dpanagou@umich.edu}}%
\thanks{Chris Vermillion is an Associate Professor with the Department of Mechanical Engineering, University of Michigan, Ann Arbor, MI, 48109, USA. {\tt\small cvermill@umich.edu}}%
}

\begin{document}

\maketitle
\thispagestyle{empty}
\pagestyle{empty}

\begin{abstract}
Sustainably powered mobile robots enable long-duration observational missions by harvesting energy from the environment. However, effective autonomy requires the robots to make decisions as to how to allocate this energy to maximize information quality over time. In this paper, we study a persistent (infinite-horizon) speed trajectory optimization problem that is posed as a periodic optimal control problem in which a robot repeatedly traverses a fixed closed path, coupling a super-linear propulsion cost with dynamic information gain and decay. Using the information-theoretic quantity of \textit{clarity} as an information metric, we first derive the necessary conditions for optimal control. We then derive a number of specialized follow-on properties, which characterize (i) the nature of the optimal control when points on the path are of varying importance and when they are of equal importance, along with (ii) a high-speed information limit that is approached when energy reserves approach infinity. We illustrate the information-optimal speed trajectories on a simple example in order to showcase key trends that are followed under the optimized speed profile. This is followed by a simulation study based on the persistent characterization of salinity in the Chesapeake Bay, illustrating the practical merit of applying the results from this work.
\end{abstract}

\section{Introduction} \label{sec:introduction}
The collection of spatiotemporally varying data is critical for environmental and engineering applications, from characterizing ocean currents for climate modeling~\cite{muglia2020gulf, mao2023data} to mapping wind speeds for offshore wind farm siting~\cite{abdelsalam2018optimization, cetinay2017optimal}. Because these tasks require spatially expansive and temporally persistent measurements, sustainably powered mobile robots---such as solar-powered autonomous surface vessels (ASVs)~\cite{seatrac} and wind-driven sailing drones~\cite{gentemann2020saildrone}---have become increasingly attractive. 

Unlike traditional finite-horizon missions, persistently powered robots must continuously balance \textit{energy inflow} (harvesting) with \textit{energy outflow} (propulsion and sensing). Similarly, a central premise behind persistent missions is the presence of \textit{information gain} resulting from observations and \textit{information decay} resulting from locations going unvisited over time. We view this interaction through a coupled ``reservoir'' lens (Fig.~\ref{fig:reservoir-diagram}). Harvested energy fills a physical battery reservoir, and this energy is converted into information, which possesses its own virtual reservoir. Because environmental information continuously decays due to forecast drift or measurement staleness~\cite{smith2011persistent}, long-term monitoring requires the robot to repeatedly revisit and refresh informative regions.

\begin{figure}
    \centering
    \includegraphics[width=\columnwidth]{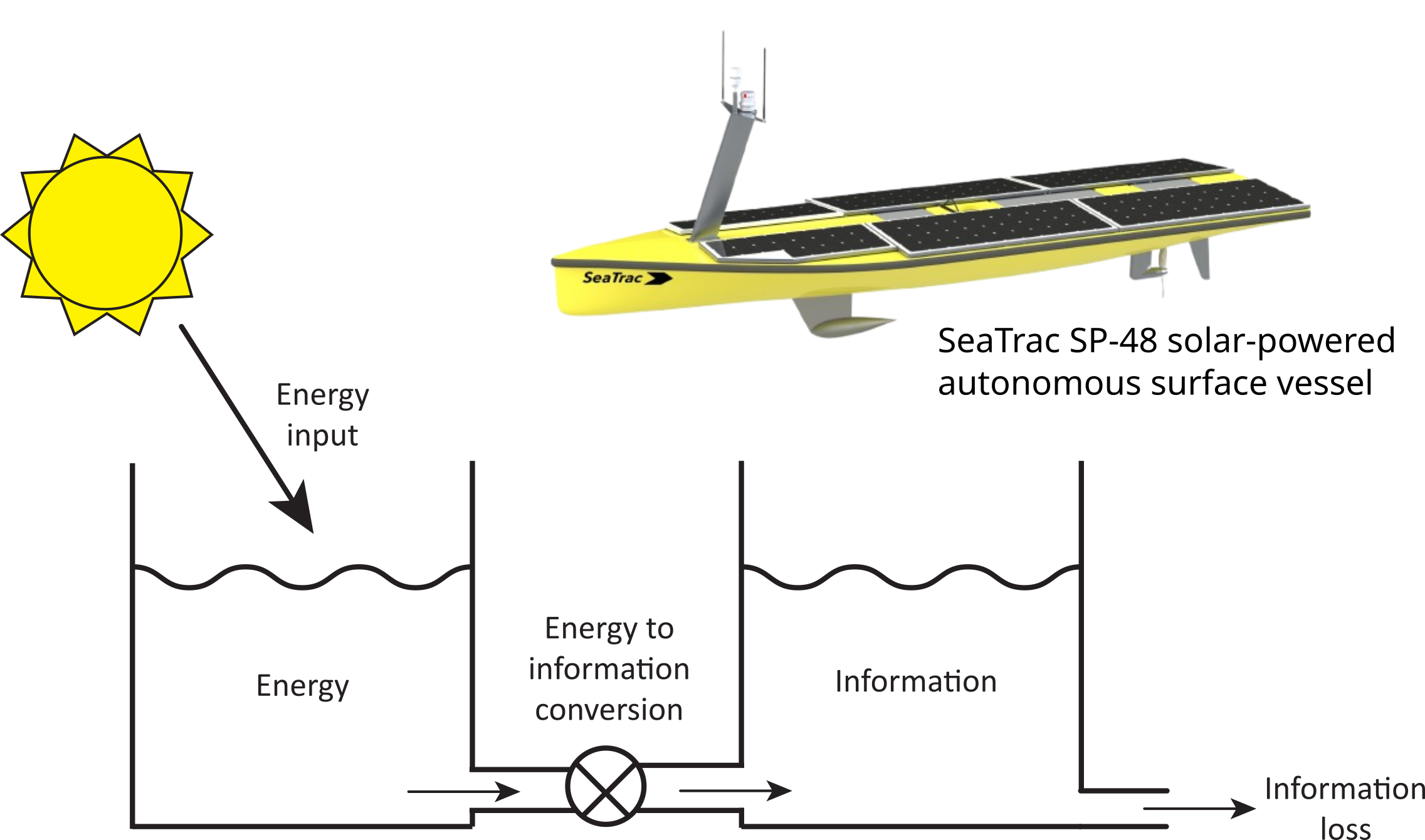}
    \caption{Sustainably powered mobile robots, such as the pictured SeaTrac SP-48 solar-powered autonomous surface vessel (image credit: SeaTrac; used with permission), can be modeled as coupled energy and information reservoirs. Energy is harvested from the environment and expended on propulsion and sensing. Sensing converts energy into information, which accumulates but also decays over time, and this conversion is regulated by the robot's speed profile and represented as a virtual valve between the two reservoirs.}
    \label{fig:reservoir-diagram}
\end{figure}

While a large body of literature addresses informative path planning~\cite{bai2021information, popovic2024learning} and energy-aware routing~\cite{plonski2013energy, diller2023energy}, these frameworks often model missions as resource-constrained traveling salesperson problems~\cite{vansteenwegen2011orienteering}. Such routing formulations treat sensing as a discrete event collected at isolated nodes. Conversely, continuous approaches like ergodic search optimize trajectories over time-varying information maps~\cite{dong2025time}, but typically rely on finite-horizon optimization formulations. However, for a mobile sensor accruing continuous rewards, a discrete objective is insufficient, since a measurement taken at any point along a path affects not only the evolution of information at that point on the path but also the evolution of information at every other location. 

To explicitly capture these dynamics, we leverage the quantity of \textit{clarity}~\cite{agrawal2023sensor}, which describes the state of certainty about a particular point at a particular time and evolves according to an ODE containing both information gain and decay terms. Coupled with a standard nonlinear propulsion model (e.g., drag power proportional to speed cubed~\cite{fossen2011handbook}), the fundamental control problem shifts from discrete routing to computing a continuous speed trajectory that maximizes a reward functional that characterizes information over the entire path at each point in time.

Specifically, we consider a mobile robot that is tasked with optimizing its speed profile over a prescribed path, given limited available power, (generally) unequal importance allocated to certain portions of the path relative to others, and periodicity constraints on final vs. initial energy (robot state of charge) and clarity (information) states. The theoretical results center around the derivation and interpretation of necessary conditions for optimality under this problem formulation. After deriving these conditions, we present several properties associated with the nature of the optimal control strategy, effect of unequal importance weighting along the mission path, and limiting performance with respect to both the importance weighting and available power.

The theoretical results in this work are complemented by two simulation studies. The first study represents an illustrative ``toy'' example that shows how the robot's speed profile varies according to the importance of information at different points along the path. The second study focuses on using a solar-powered autonomous surface vessel (ASV) to characterize salinity information in the Chesapeake Bay. Results of the simulation study illustrate the value of modifying the ASV's speed profile by solving the optimization problem presented in this work.

In summary, the contributions of this paper are as follows:
\begin{enumerate}
    \item A modeling framework coupling clarity dynamics with nonlinear energy expenditure;
    \item An optimal control formulation for periodic sensing missions derived via Pontryagin's Maximum Principle, yielding necessary conditions for optimality;
    \item Analytical results establishing several properties of the optimal control solution;
    \item Numerical validation via IPOPT, demonstrating optimal speed profiles that slow in high-information regions and accelerate through low-value regions;
    \item A simulation case study for salinity monitoring in the Chesapeake Bay, showcasing the utility of the work in a realistic application domain.
\end{enumerate}

\section{Preliminaries}\label{sec:preliminaries}
\subsection{Information Model}\label{sec:preliminaries:info-model}
In this work, we characterize information using \textit{clarity}, an information-theoretic metric introduced in~\cite{agrawal2023sensor}. Clarity is a normalized mapping $q: \mathcal{D} \times \mathbb{R}_{\geq 0} \to [0,1]$, where $q = 0$ denotes maximum uncertainty and $q=1$ denotes perfect knowledge in a domain $\mathcal{D}$. For a distribution characterizing a scalar field of interest at spatial location $p \in \mathcal{D} \subset \mathbb{R}^2$ and time $t \in \mathbb{R}_{\geq 0}$, clarity is related to the state estimation variance (or posterior covariance) $\Sigma(p,t) \in \mathbb{R}_{\geq 0}$ by:
\begin{equation}\label{eqn:clarity-covariance}
    q(p,t) = \frac{1}{1 + \Sigma(p,t)}.
\end{equation}

As derived in~\cite{agrawal2023sensor}, the clarity dynamics at a specific location $p$ follow the ordinary differential equation:
\begin{equation}\label{eqn:clarity-dynamics}
    \dot q(p, t) = S(p_r(t), p) (1 - q(p, t))^2 - \alpha q(p, t)^2,
\end{equation}
where $\alpha > 0$ is a decay constant representing natural information loss due to environmental dynamics, and $p_r(t) \in \mathcal{D}$ denotes the robot's position at time $t$. The term $S: \mathcal{D} \times \mathcal{D} \to \mathbb{R}_{\geq 0}$ represents the spatial sensing profile. Often (and in the case studies of this work), $S(p_r, p)$ is modeled as a Gaussian function of the distance between the robot and the target location $p$~\cite{agrawal2023sensor}:
\begin{equation}\label{eqn:sensing_model}
    S(p_r, p) = S_0 \exp\left(-\frac{\|p_r - p\|^2}{2\sigma^2}\right),
\end{equation}
where $S_0 \in \mathbb{R}_{\geq 0}$ is the maximum sensing rate and $\sigma \in \mathbb{R}_{>0}$ is a spatial length scale encoding the sensor's effective range.

\subsection{Energy Model}\label{sec:preliminaries:energy-model}
In general, we express the robot's energy model as:
\begin{equation}\label{eqn:general-energy-dynamics}
    \dot{b}(t) = P_{\mathrm{in}}(t) - P_{\mathrm{out}}(u(t)),
\end{equation}
where $b(t)$ is the battery state of charge at time $t$, $P_{\mathrm{in}}(t)$ is the time-varying power input from energy harvesting, and $P_{\mathrm{out}}(u(t))$ is the power consumption as a function of the robot's speed $u(t)$. The specific form of $P_{\mathrm{out}}$ depends on the robotic platform and its propulsion characteristics.

For the analysis in this work, we assume $P_{\mathrm{out}} \in C^2$, $ \frac{\partial P_{\mathrm{out}}}{\partial u}(u) > 0$, and $\frac{\partial^2 P_{\mathrm{out}}}{\partial u^2}(u) > 0$ for all $u > 0$. This amounts to assuming that the power required grows monotonically and super-linearly with speed, which is a standard result for hydrodynamic and aerodynamic drag power~\cite{fossen2011handbook}. This assumption does not preclude the existence of additional ``hotel'' loads (e.g., sensing and computing loads) that are constant with respect to speed. Additionally, for the analysis in this work, we assume that the energy reserves are always sufficient to provide energy (i.e, $b(t) \geq 0$ for all times $t$).

Note that $\lVert \dot{p}_r(t) \rVert = u(t)$, i.e., the robot's speed is the norm of its velocity vector, thus linking the energy and information models through the robot's trajectory. Throughout the paper, we use $u(\cdot)$ to denote speed, $s$ to denote distance traversed along the path, and $t$ to denote time.

\section{Optimization Formulation and Theoretical Results}\label{sec:energy-optimization}
\subsection{Optimal Control Problem Formulation}\label{sec:eoc:ocp-formulation}

We consider a persistent monitoring mission where the robot traverses a closed, fixed path $\mathcal{P}$ of length $L$, parameterized by distance along the path $s \in [0, L]$. The control input at each time $t$ is the robot's speed $u(t)$. Further, we use the variable $u$ to represent the robot's entire speed trajectory. When working in the spatial domain, we use $u(s) \triangleq u(t(s))$ to denote the same control as a function of position. To allow the optimization full freedom in selecting the speed profile, we pose the optimization as a \textit{free final time} problem, where the period of one lap, $T_{\text{lap}}$, is functionally dependent on the speed profile:
\begin{equation}\label{eq:tlap_def}
    T_{\text{lap}}(u) = \int_{0}^{L} \frac{1}{u(s)} \, ds.
\end{equation}

To derive a set of necessary conditions consisting of a finite number of differential equations and boundary conditions, we transcribe the problem into a finite-dimensional state space. Note that this is also representative of how the numerical optimization is set up. We discretize the continuous path $\mathcal{P}$ into $N$ representative nodes, as shown in Fig.~\ref{fig:discretization}. We then define the clarity at node $j$ as $q_j(t) \equiv q(x_j, t)$, where $x_j$ is the position of the node, its weight as $\gamma_j \equiv \gamma(x_j) > 0, \forall j$, and the localized sensing intensity as $S_j(s) \equiv S(s, x_j)$, where $s$ is the position of the robot.

\begin{figure}
    \centering
    \includegraphics[width=0.70\columnwidth]{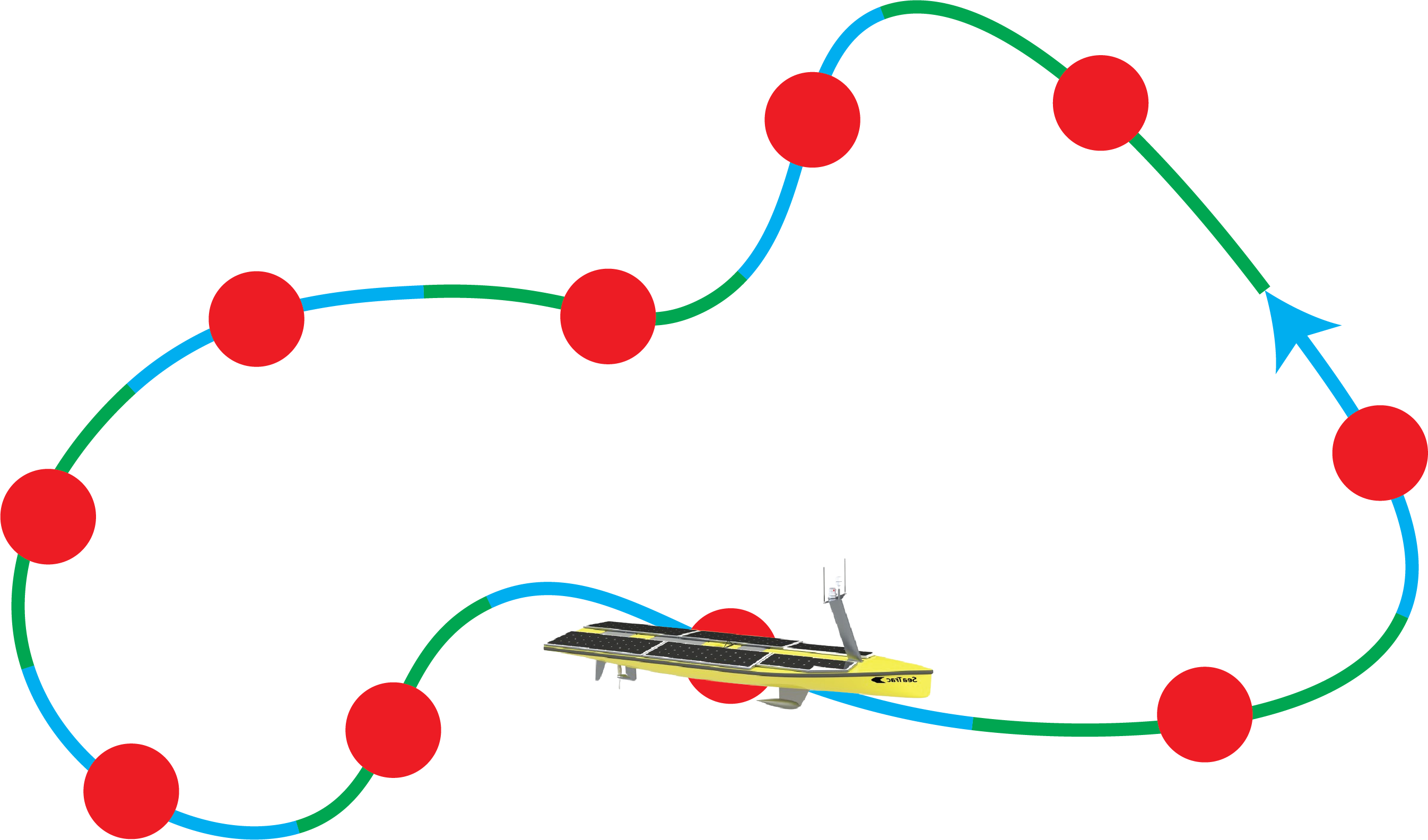}
    \caption{Illustration of the finite-dimensional transcription. The continuous path $\mathcal{P}$ is discretized into $N$ nodes (represented by red dots), each with its own clarity state $q_j(t)$ and weight $\gamma_j$.}
    \label{fig:discretization}
\end{figure}

This yields the following optimal control problem:
\begin{subequations}\label{eqn:eoc:ocp}
    \begin{align}
        \max_{u} \quad & R = \frac{1}{T_{\text{lap}}(u)} \int_{0}^{T_{\text{lap}}(u)} \sum_{j=1}^N \gamma_j q_j(t) \, dt \label{eqn:eoc:ocp:obj}              \\
        \text{s.t.} \quad & \dot{s} = u(t) \label{eqn:eoc:ocp:dynamics}\\
        & s(0)=0,\; s(T_{\text{lap}}(u)) = L \label{eqn:eoc:ocp:path-bc}\\
        & \dot{q}_j(t) = S_j(s)(1 - q_j(t))^2 - \alpha q_j(t)^2, \nonumber\\ &\quad \forall j \in \{1, \dots, N\} \label{eqn:eoc:ocp:clarity-dynamics} \\
        & q_j(0) = q_j(T_{\text{lap}}(u)), \quad \forall j \in \{1, \dots, N\} \label{eqn:eoc:ocp:periodicity} \\
        & \frac{1}{T_{\text{lap}}(u)} \int_{0}^{T_{\text{lap}}(u)} P_{\mathrm{out}}(u(t)) \, dt \leq P_{\mathrm{in,avg}}, \label{eqn:eoc:ocp:energy-constraint}
    \end{align}
\end{subequations}
where constraint~\eqref{eqn:eoc:ocp:periodicity} enforces the periodic nature of the mission, ensuring that the clarity state at the end of the lap matches that at the beginning. Similarly, constraint~\eqref{eqn:eoc:ocp:energy-constraint} ensures that the robot only consumes energy that is made available over each period, thereby enabling indefinite periodic operation. These periodicity constraints capture the persistent nature of the mission while preserving a finite-time optimal control formulation. The discretization of the path into a finite number ($N$) of nodes allows for subsequent theory to be derived for a finite-dimensional state space representation, while being applicable to arbitrarily large $N$.

\subsection{Conditions for Optimality via PMP}\label{sec:eoc:pmp}

\subsubsection{Necessary Conditions}

We derive the continuous-time necessary conditions using Pontryagin's Maximum Principle (PMP). We can rewrite the time-averaged objective of~\eqref{eqn:eoc:ocp:obj} as:
 \begin{equation}
     R = \frac{1}{L} \int_0^{T_{\text{lap}}(u)} u_{\mathrm{ave}}(u) \left(\sum_{j=1}^N \gamma_j q_j(t)\right) dt.
 \end{equation}
 \noindent where:
 \begin{equation}
 u_{\mathrm{ave}}(u) = \frac{L}{T_{\mathrm{lap}}(u)} = \frac{\int_0^{T_{\text{lap}}(u)} u(t) dt}{T_{\text{lap}}(u)}.
 \end{equation}
 \noindent Note that $u_{\mathrm{ave}}(u)$ is a single scalar determined by the entire trajectory $u(\cdot)$, it does not vary over the course of a lap.

We write the full state vector as: 
\begin{equation}
\mathbf{x}(t) = [\begin{array}{ccccc} s(t) & q_1(t) & \dots & q_N(t) & b(t) \end{array}]^{T},
\end{equation} 
\noindent with $\dot{b} = P_{\mathrm{in,avg}} - P_{\mathrm{out}}(u(t))$, the periodicity constraint $q_j(0)=q_j(T_{\text{lap}})$, and the constraints $b(0) \leq b(T_{\text{lap}})$, $s(0)=0$, and $s(T_{\text{lap}})=L$. We define the costate vector as: 
\begin{equation}
\boldsymbol{\lambda}(t) = [\begin{array}{ccccc} \lambda_s & \lambda_{q_1} & \dots & \lambda_{q_N} & \lambda_b \end{array}]^{T},
\end{equation}

\noindent and formulate the Hamiltonian $\mathcal{H}(\mathbf{x}, u, \boldsymbol{\lambda}, t)$:
\begin{align}\label{eqn:pmp:hamiltonian}
	\mathcal{H} &= u_{\mathrm{ave}}(u)\sum_{j=1}^N \gamma_j q_j(t) + \lambda_s(t) u(t) \nonumber\\
    & + \lambda_b(t) \left( P_{\mathrm{in,avg}} - P_{\mathrm{out}}(u(t)) \right) \nonumber \\
    & + \sum_{j=1}^N \lambda_{q_j}(t) \left[ S_j(s)(1 - q_j(t))^2 - \alpha q_j(t)^2 \right].
\end{align}

The evolution of the costates is dictated by the canonical equations $\dot{\boldsymbol{\lambda}} = -\nabla_{\mathbf{x}} \mathcal{H}$:

\begin{subequations}\label{eqn:pmp:costate-dynamics}
	\begin{align}
    		\dot{\lambda}_{s} &= - \left[ \sum_{j=1}^{N} \frac{\partial S_{j}(s)}{\partial s} \lambda_{q_j} (1 - q_j)^2 \right] \label{eqn:pmp:position-costate-dynamics} \\
    		\dot{\lambda}_{b} &= 0 \label{eqn:pmp:battery-costate-dynamics}\\
		\dot{\lambda}_{q_j} &= - u_{\mathrm{ave}}(u) \gamma_j - \lambda_{q_j} \left(-2 S_{j}(s)(1-q_j) - 2 \alpha q_j \right). \label{eqn:pmp:clarity-costate-dynamics}
	\end{align}
\end{subequations}
Notice that because the energy state $b(t)$ does not explicitly appear in the Hamiltonian, its costate dynamics are identically zero, and $\lambda_b$ is constant.

To determine the optimal unconstrained control trajectory $u^*$, PMP requires maximizing the Hamiltonian with respect to $u(t)$ at all times. It is important to note that although $u_{\mathrm{ave}}(u)$ clearly depends on the full speed trajectory $u$, its first derivative with respect to the value of $u(t)$ at any point in time is equal to zero. Consequently, the following stationarity condition results from PMP:

\begin{align}
	\frac{\partial \mathcal{H}}{\partial u} &= u_{\mathrm{ave}} \left(\lambda_s - \lambda_b \frac{\partial P_{\mathrm{out}}(u)}{\partial u}\right) = 0 \nonumber \\
						&\implies \frac{\partial P_{\mathrm{out}}(u^*)}{\partial u} = \frac{\lambda_s}{\lambda_b}. \label{eqn:pmp:stationarity}
\end{align}

The stationarity condition in~\eqref{eqn:pmp:stationarity} is a necessary condition for optimality, and it defines the unconstrained optimal control law $u^*(t)$ as a function of the costates $\lambda_s$ and $\lambda_b$, which follow the dynamics in~\eqref{eqn:pmp:costate-dynamics}. This relationship provides a direct link between the energy dynamics and the information state, thus representing the ``information value of energy'' in the context of the optimal control problem.

\subsection{Theoretical Properties and Analysis of Optimal Solutions}

We now establish several properties of the optimal control law that give insight into its structure and support the subsequent numerical analysis.

\subsubsection{Sufficiency Considerations}

While the stationarity condition in~\eqref{eqn:pmp:stationarity} and the associated boundary conditions establish necessary conditions for optimality, the overall problem is non-convex: The clarity dynamics in~\eqref{eqn:eoc:ocp:clarity-dynamics} introduce a strictly nonlinear equality constraint (via the squared clarity terms and the nonlinear sensing profile $S_j(s)$), violating the affine-constraint requirement for convexity~\cite[Sec 4.2]{boyd2004convex} and prohibiting a blanket sufficiency result from being derived. We can nonetheless establish useful properties arising from the solution of the optimal control problem.

We first define the admissible control set and establish the intuitive result that the optimal reward is a non-decreasing function of available power.

\begin{definition}[Admissible Speed Control Set]\label{def:admissible-set}
For a fixed initial battery state $b(0)=b_0$ and average input power $P$, let $\mathcal{U}(P)$ denote the set of admissible speed policies $u(\cdot)$ satisfying the constraints in~\eqref{eqn:eoc:ocp}. 
\end{definition}

\begin{proposition}[Monotonicity of the Value Function]\label{thm:value-monotonicity}
Let
\begin{equation}
    V(P)\triangleq\sup_{u\in\mathcal{U}(P)}R(u)
\end{equation}
denote the optimal time-averaged clarity under available average input power $P$. For any $P_2>P_1>0$,
\begin{equation}
    V(P_2)\geq V(P_1).
\end{equation}
Consequently, wherever the derivative exists,
\begin{equation}
    \frac{\partial V}{\partial P_{\mathrm{in,avg}}}\geq0.
\end{equation}
\end{proposition}

\begin{proof}
Let $u_1(\cdot)\in\mathcal{U}(P_1)$ be arbitrary, with corresponding battery trajectory $b_1(t)$. Applying the same control under $P_2$ yields
\begin{equation}
    b_2(t)=b_1(t)+(P_2-P_1)t.
\end{equation}
Hence, $b_2(t)\geq b_1(t)\geq0$, and
\begin{equation}
\begin{aligned}
    b_2(T_{\mathrm{lap}})-b_0
    ={}&b_1(T_{\mathrm{lap}})-b_0\\
    &+(P_2-P_1)T_{\mathrm{lap}}
    \geq0.
\end{aligned}
\end{equation}
The path and clarity trajectories are unchanged, so $u_1(\cdot)\in\mathcal{U}(P_2)$. Therefore,
$\mathcal{U}(P_1)\subseteq\mathcal{U}(P_2)$, which implies
$V(P_2)\geq V(P_1)$ and the stated derivative inequality wherever $V$ is differentiable.
\end{proof}

Note that:
\begin{equation}\label{eqn:energy-costate-sensitivity}
    \lambda_b = \frac{\partial V}{\partial b} = \frac{1}{T_{\mathrm{lap}}}\frac{\partial V}{\partial P_{\mathrm{in,avg}}}.
\end{equation}

\noindent With this, we now formalize the uniqueness of the pointwise Hamiltonian maximizer under strict satisfaction of the aforementioned monotonicity property.

\begin{proposition}[Unique Pointwise Control Mapping]\label{prop:uniqueness}
Assume that $P_{\mathrm{out}}(u)$ is strictly convex for $u>0$ and that the system strictly satisfies the monotonicity property of Proposition \ref{thm:value-monotonicity} (i.e., $\frac{\partial V}{\partial P_{\mathrm{in,avg}}}>0$). Equivalently, by~\eqref{eqn:energy-costate-sensitivity}, $\lambda_b>0$. Then the stationarity condition derived in Section~\ref{sec:eoc:pmp} defines a unique global pointwise maximizer of the Hamiltonian with respect to $u$.
\end{proposition}

\begin{proof}
The stationarity condition defines $u^*(t)$ according to
\begin{equation}
    \frac{\partial P_{\mathrm{out}}(u^*)}{\partial u}
    =\frac{\lambda_s(t)}{\lambda_b}.
\end{equation}
The second derivative of the Hamiltonian with respect to $u$ is
\begin{equation}\label{eqn:pmp:legendre-clebsch}
    \frac{\partial^2\mathcal{H}}{\partial u^2}
    =-\lambda_b
    \frac{\partial^2P_{\mathrm{out}}(u)}{\partial u^2}.
\end{equation}
Strict convexity gives
$\frac{\partial^2P_{\mathrm{out}}(u)}{\partial u^2}>0$, while
$\lambda_b>0$ by assumption. Therefore,
\begin{equation}
    \frac{\partial^2\mathcal{H}}{\partial u^2}<0,
    \qquad \forall u>0.
\end{equation}
Thus, the Hamiltonian is strictly concave with respect to $u$, and $u^*(t)$ is its unique global pointwise maximizer along the optimal state-costate trajectory.
\end{proof}

While this does not guarantee sufficiency, it does rule out the possibility of the solution being a minimizer.

\subsubsection{Continuity of the Optimal Control}

We now establish that under reasonable assumptions on the sensing profiles and power draw model, the optimal unconstrained speed trajectory $u^*(t)$ is a continuous function of time over the interval $t \in [0, T_{\text{lap}}]$. This property is important for practical implementation, as it avoids abrupt speed changes that could be physically infeasible.

\begin{proposition}[Continuity of the Optimal Speed]\label{prop:continuity}
    Assume the position-dependent sensing profiles $S_j(s)$ are continuously differentiable with respect to $s$, and the power draw model $P_{\mathrm{out}}(u)$ is strictly convex in $u$. Then, the optimal unconstrained speed trajectory $u^*(t)$ is a continuous function of time over the interval $t \in [0, T_{\text{lap}}]$.
\end{proposition}

\begin{proof}
    Recall the spatial costate dynamics in~\eqref{eqn:pmp:position-costate-dynamics}:
    \begin{equation}
        \dot{\lambda}_s(t) = -\sum_{j=1}^N \lambda_{q_j}(t) \frac{\partial S_j(s(t))}{\partial s} (1-q_j(t))^2.
    \end{equation}
    Because the state trajectories $q_j(t)$ and $s(t)$ are continuous in time (due to their finite time derivatives), and all $S_j(s)$ are continuously differentiable with respect to $s$, the right-hand side of this differential equation is continuous. Therefore, the spatial costate $\lambda_s(t)$ is continuous in time.

    Because $P_{\mathrm{out}}(u)$ is strictly convex by assumption ($\frac{\partial^2 P_{\mathrm{out}}(u)}{\partial u^2} > 0$), its first derivative $\frac{\partial P_{\mathrm{out}}(u)}{\partial u}$ is a continuous and strictly monotonically increasing function of $u$. Consequently, by the Inverse Function Theorem, its inverse mapping $\left(\frac{\partial P_{\mathrm{out}}}{\partial u}\right)^{-1}$ exists and is continuous, so~\eqref{eqn:pmp:stationarity} can be solved explicitly for the optimal control:
    \begin{equation}
        u^*(t) = \left(\frac{\partial P_{\mathrm{out}}}{\partial u}\right)^{-1} \left( \frac{\lambda_s(t)}{\lambda_b} \right).
    \end{equation}
    Since $\lambda_b$ is a non-zero constant, the argument to the inverse function is continuous. The composition of the continuous inverse mapping $\left(\frac{\partial P_{\mathrm{out}}}{\partial u}\right)^{-1}$ with the continuous costate trajectory $\lambda_s(t)$ guarantees that $u^*(t)$ is continuous. 
\end{proof}

\subsubsection{Speed Trajectory Analysis}\label{sec:eoc:speed-trajectory}
An important element of our theoretical analysis lies in establishing an understanding of how the optimal speed trajectory must vary in the face of unequal importance weighting along the path. Intuition would suggest that the optimal robot behavior will involve slowing down in locations where the importance weighting is higher, in order to allocate more time sensing these regions. The accompanying theory confirms this intuition and also quantifies the mechanism and extent to which (and by which) this occurs.

Before analyzing the speed trajectory itself, we establish a property of the clarity costates $\lambda_{q_j}(t)$ that the subsequent analysis relies on.

\begin{lemma}[Positivity of the Clarity Costate]\label{lem:costate-positivity}
    Along any optimal extremal trajectory, $\lambda_{q_j}(t) > 0$ for all $t \in [0, T_{\mathrm{lap}}]$ and all $j \in \{1, \dots, N\}$.
\end{lemma}
\begin{proof}
    Since $S_j(s) \ge 0$, $\alpha > 0$, and $q_j(t) \in (0,1)$ by the definition of clarity, the coefficient $A_j(t) \triangleq 2\left[S_j(s(t))(1-q_j(t)) + \alpha q_j(t)\right]$ in~\eqref{eqn:pmp:clarity-costate-dynamics} is non-negative (and strictly positive whenever $S_j(s(t))>0$). We can rewrite~\eqref{eqn:pmp:clarity-costate-dynamics} in the form:
    \begin{equation}\label{eqn:costate-linear-ode}
        \dot{\lambda}_{q_j}(t) - A_j(t) \lambda_{q_j}(t) = -u_{\mathrm{ave}}(u)\gamma_j,
    \end{equation}
    with $u_{\mathrm{ave}}(u)\gamma_j > 0$ constant. Let $\mu_j(t) \triangleq \exp\left(-\int_0^t A_j(\tau)\,d\tau\right)$ denote the integrating factor; since $A_j \ge 0$, $\mu_j(t)$ is non-increasing with $\mu_j(0) = 1 \ge \mu_j(T_{\mathrm{lap}}) > 0$. Solving~\eqref{eqn:costate-linear-ode} subject to the periodicity constraint $\lambda_{q_j}(0) = \lambda_{q_j}(T_{\mathrm{lap}})$ yields the closed-form expression
    \begin{equation}\label{eqn:costate-closed-form}
        \begin{aligned}
            \lambda_{q_j}(t) = \frac{u_{\mathrm{ave}}(u)\gamma_j}{\mu_j(t)} \biggl[{}&
                \frac{\mu_j(T_{\mathrm{lap}})}{1-\mu_j(T_{\mathrm{lap}})}
                \int_0^{t} \mu_j \, d\tau \\
                &{}+ \frac{1}{1-\mu_j(T_{\mathrm{lap}})}
                \int_{t}^{T_{\mathrm{lap}}} \mu_j \, d\tau
            \biggr].
        \end{aligned}
    \end{equation}
    Since $\mu_j(T_{\mathrm{lap}}) \in (0,1)$, both bracketed integrals are non-negative with strictly positive coefficients (strictly positive for $t<T_{\mathrm{lap}}$, and strictly positive at $t=T_{\mathrm{lap}}$ by periodicity), and the prefactor $u_{\mathrm{ave}}(u)\gamma_j/\mu_j(t)$ is strictly positive. Hence $\lambda_{q_j}(t) > 0$ for all $t \in [0, T_{\mathrm{lap}}]$.
\end{proof}

From this result, we can see that $\lambda_{q_j}(t)$ is directly proportional to $\gamma_j$ of the corresponding node, and is strictly positive for all $t \in [0, T_{\mathrm{lap}}]$ and all $j \in \{1, \dots, N\}$. The result enables us to establish a quantitative understanding of how the robot's optimal speed trajectory exhibits acceleration and deceleration due to changes in information importance weighting along the path. Prior to presenting our main result, it will be useful to define a quantity termed the \textit{weighted spatial sensing sensitivity}, given by:

\begin{equation}
\xi(\lambda_{q_j(t)},S_j(s(t)),q_j(t)) = \sum_{j=1}^N \lambda_{q_j}(t)
        \frac{\partial S_j(s(t))}{\partial s} (1-q_j(t))^2.
\end{equation}
\noindent This quantity reflects a weighted sum -- across all the nodes in the path -- of the sensitivities of the nodal sensing functions to the robot's position along the path. With that quantity defined, we are now in a position to present the following result:

\begin{proposition}[Information-Driven Acceleration]\label{prop:non-uniform-speed}
    Along an optimal extremal trajectory in a spatially heterogeneous environment, the sign of the time derivative of the optimal speed profile $\dot{u}^*(t)$ opposes the sign of the weighted spatial sensing sensitivity, according to:
    \begin{equation}
    \dot{u}^*(t) = -\frac{1}{\lambda_b \frac{\partial^2 P_{\mathrm{out}}(u^*(t))}{\partial u^2}} \xi(\lambda_{q_j(t)},S_j(s(t)),q_j(t)).
    \end{equation}
    
\end{proposition}

\vspace{2mm}

\begin{proof}
    Differentiating the stationarity condition~\eqref{eqn:pmp:stationarity} with respect to $t$ yields:
    \begin{equation}
        \frac{\partial^2 P_{\mathrm{out}}(u^*(t))}{\partial u^2} \dot{u}^*(t) = \frac{\dot{\lambda}_s(t)}{\lambda_b}.
    \end{equation}
    Substituting the costate dynamics for $\dot{\lambda}_s(t)$ established by the canonical equations isolates the control derivative:
    \begin{align}
        \frac{\partial^2 P_{\mathrm{out}}(u^*(t))}{\partial u^2}
        \dot{u}^*(t)
        &= -\frac{1}{\lambda_b} \frac{\partial \mathcal{H}}{\partial s}
        \nonumber \\
        &= -\frac{1}{\lambda_b} \sum_{j=1}^N \lambda_{q_j}(t)
        \frac{\partial S_j(s(t))}{\partial s} (1-q_j(t))^2.
    \end{align}
    Rearranging this expression isolates the time derivative of the optimal control:
    \begin{equation} \label{eq:sensitivity_summation}
        \dot{u}^*(t) = -\frac{1}{\lambda_b \frac{\partial^2 P_{\mathrm{out}}(u^*(t))}{\partial u^2}} \sum_{j=1}^N \lambda_{q_j}(t) \frac{\partial S_j(s(t))}{\partial s} (1-q_j(t))^2.
    \end{equation}
By the strict convexity of the power draw model, $\frac{\partial^2 P_{\mathrm{out}}(u^*(t))}{\partial u^2} > 0$. Under the strictly energy-constrained assumption of Proposition~\ref{prop:uniqueness}, we have that $\lambda_b > 0$. Therefore, the sign of $\dot{u}^*(t)$ will always oppose the sign of $\xi(\lambda_{q_j(t)},S_j(s(t)),q_j(t))$.
\end{proof}

An important interpretation follows from the aforementioned result. Based on the fact that Lemma~\ref{lem:costate-positivity} guarantees that $\lambda_{q_j}(t) > 0$, and given that $(1-q_j(t))^2 \ge 0$, each individual term in the summation of \eqref{eq:sensitivity_summation} contributes to an acceleration whose sign opposes that of $\frac{\partial S_j(s)}{\partial s}$. However, the terms do not carry equal weight. Specifically, Lemma~\ref{lem:costate-positivity} shows that the costate $\lambda_{q_j}$ is linearly proportional to $\gamma_j$, meaning that nodes with larger importance weighting ($\gamma_j$) carry more weight in the summation. As the robot approaches a node of relative importance, $\frac{\partial S_j}{\partial s} > 0$ for that node. Because of its relative importance, that element carries outsized influence on the aggregate sum, likely rendering $\dot{u}^*(t) < 0$ and compelling the robot to decelerate. Conversely, as the robot departs this high-value node and sensing potential diminishes ($\frac{\partial S_j}{\partial s} < 0$), the robot is compelled to accelerate. 

The following result confirms that in the special case of a uniformly important environment, the optimal speed trajectory is indeed constant:
\begin{corollary}[Constant Speed in Uniform Environments]\label{cor:uniform-speed}
    In a uniform, shift-invariant environment, where $\gamma_j = \gamma_k$ for all $j,k \in \{1, \ldots, n\}$, the aggregate spatial gradient vanishes ($\frac{\partial \mathcal{H}}{\partial s} = 0$), dictating a constant speed profile $u^*(t) = C$.
\end{corollary}

\begin{proof}
    By shift-invariance, $\frac{\partial \mathcal{H}}{\partial s} = 0$. Substituting this into the spatial costate dynamics yields $\dot{\lambda}_s = 0$, rendering the right side of the stationarity condition a constant. Due to the strict convexity of $P_{\mathrm{out}}$, the marginal power mapping is injective, forcing the speed $u^*(t)$ to evaluate to a strictly constant scalar $C$.
\end{proof}

\subsubsection{High-Speed Information Limit}\label{sec:eoc:high-speed}
While it is reasonable to assume (and is in fact confirmed through case studies) that the optimal reward function value increases as more energy becomes available, it turns out that this growth is not unbounded. Due to the nonlinear nature of the clarity dynamics, the time-averaged clarity ultimately approaches a fundamental saturation limit, as delineated in the following proposition:

\begin{proposition}[High-Speed Saturation]\label{lem:limiting_clarity}
    Consider the clarity dynamics in~\eqref{eqn:eoc:ocp:clarity-dynamics}. For a constant traversal speed $u$, as $u \to \infty$, the time-averaged clarity $\bar{q}_j(u)$ at any node $j$ converges to a finite limit:
    \begin{equation}
        q_{j,\infty} = \lim_{u \to \infty} \bar{q}_j(u) = \frac{\sqrt{\bar{S}_j}}{\sqrt{\bar{S}_j} + \sqrt{\alpha}},
    \end{equation}
    where $\bar{S}_j = \frac{1}{L}\int_0^L S_j(s)\,ds$ is the spatial average of the sensing profile over the path.
\end{proposition}
\begin{proof}
    We transition from the temporal domain to the spatial domain. Using the relation $s = u t$, we have $ds = u \, dt$. Applying this change of independent variable to the clarity dynamics yields:
    \begin{equation} \label{eqn:spatial_ode}
        \frac{dq_j}{ds} = \frac{1}{u} \left[ S_j(s)(1-q_j)^2 - \alpha q_j^2 \right].
    \end{equation}
    Let $\varepsilon = 1/u$. Eq.~\eqref{eqn:spatial_ode} is then expressed in the standard form for the Method of Averaging~\cite[Sec. 10.4]{khalil2002nonlinear}:
    \begin{equation}
        \frac{dq_j}{ds} = \varepsilon f(s, q_j), \quad \text{where} \quad f(s, q_j) = S_j(s)(1-q_j)^2 - \alpha q_j^2.
    \end{equation}
    The vector field $f(s, q_j)$ is $L$-periodic in $s$ and continuously differentiable with respect to $q_j$. According to averaging theory, as $\varepsilon \to 0$ (which corresponds to $u \to \infty$), the trajectory $q_j(s; \varepsilon)$ converges uniformly to the solution of the autonomous averaged system:
    \begin{equation}
        \begin{aligned}
            \frac{d\bar{q}_j}{ds} &= \varepsilon f_{\text{avg}}(\bar{q}_j), \\
            \text{where} \quad f_{\text{avg}}(\bar{q}_j) &= \frac{1}{L} \int_0^L f(s, \bar{q}_j) \, ds.
        \end{aligned}
    \end{equation}
    Evaluating the integral yields $f_{\text{avg}}(\bar{q}_j) = \bar{S}_j(1-\bar{q}_j)^2 - \alpha \bar{q}_j^2$. Solving for equilibrium by setting $f_{\text{avg}}(q_{j,\infty}) = 0$:
    \begin{equation}
        \begin{aligned}
            \bar{S}_j(1 - q_{j,\infty})^2 &= \alpha q_{j,\infty}^2 \\
            \implies \sqrt{\bar{S}_j}(1 - q_{j,\infty}) &= \sqrt{\alpha} q_{j,\infty}.
        \end{aligned}
    \end{equation}
    Solving for the unique positive root yields $q_{j,\infty} = \frac{\sqrt{\bar{S}_j}}{\sqrt{\bar{S}_j} + \sqrt{\alpha}}$. To ensure this equilibrium is valid, we verify its exponential stability by evaluating the Jacobian of the averaged vector field:
    \begin{equation}
        \left. \frac{\partial f_{\text{avg}}}{\partial q_j} \right|_{q_{j,\infty}} = -2\bar{S}_j(1-q_{j,\infty}) - 2\alpha q_{j,\infty}.
    \end{equation}
    Since $\bar{S}_j > 0$, $\alpha > 0$, and $q_{j,\infty} \in (0, 1)$, the Jacobian is strictly negative, guaranteeing exponential stability. Thus, by Theorem 10.4 of~\cite{khalil2002nonlinear}, the $L$-periodic steady-state solution $q_j(s; u)$ converges uniformly to $q_{j,\infty}$ as $u \to \infty$, rendering the time-averaged clarity equal to this limit.
\end{proof}

In summary, in this section, we have formulated the energy-optimal control problem as a finite-dimensional optimal control problem, derived necessary conditions for optimality using PMP, and established several theoretical properties of the optimal control law. These properties include the uniqueness of the pointwise Hamiltonian maximizer, continuity of the optimal speed trajectory, the influence of spatial heterogeneity on acceleration, and the existence of a high-speed saturation limit for clarity. These insights provide a foundation for understanding the structure of optimal solutions and inform subsequent numerical analyses and simulations.

\section{Numerical Validation on an Illustrative Example}\label{sec:numerical-validation}

In this section, we present a computational study to illustrate overall trends and analyze the sensitivity of the solver and results to the choice of discretization $N$. We consider a simplified scenario where the path is fixed and the spatial information distribution is known a priori. The objective is to compute the optimal speed profile that maximizes the time-averaged clarity accumulated along the path, subject to the energy dynamics and constraints characterized by the necessary conditions in Sec.~\ref{sec:energy-optimization}.

As shown in Fig.~\ref{fig:resolution-comparison}, we define a spatially varying information distribution along a $100\,\mathrm{m}$ path, featuring high-density regions centered at $30\,\mathrm{m}$ and $80\,\mathrm{m}$. We discretize the path into $N$ segments and parameterize the speed profile as piecewise constant over each segment. We then solve the resulting finite-dimensional non-linear programming (NLP) problem defined in~\eqref{eqn:eoc:ocp} using the IPOPT solver~\cite{wachter2006implementation}. Note that this piecewise constant formulation is a numerical parameterization choice for the NLP solver; as $N \to \infty$, the step-like profile converges to the continuous optimal speed trajectory.

Regarding the vehicle constraints, we adopt the energy dynamics associated with the SeaTrac SP-48 ASV~\cite{seatrac}:
\begin{equation}\label{eq:nv:energy-dynamics}
    P_{\text{out}}(u(t)) = k_h + k_m u(t)^3,
\end{equation}
where $k_h = 10\,\mathrm{W}$ is the hotel load and $k_m = 83\,\mathrm{W \cdot (m/s)^{-3}}$ represents a lumped drag coefficient. The vehicle receives a constant average power input of $P_{\mathrm{in,avg}} = 750\,\mathrm{W}$, which is representative of the average power output from the solar panels under typical conditions.

For this implementation and hardware setup, computation time grows super-linearly with the number of segments, with the objective converging around $N = 50$ points. As shown in Fig.~\ref{fig:resolution-comparison}, the optimal speed profile consistently exhibits the behavior predicted in Section~\ref{sec:energy-optimization}: the vehicle slows down in high-information-density regions ($30\,\mathrm{m}$ and $80\,\mathrm{m}$) and accelerates through low-information regions. The optimization determines exactly how much the vehicle should accelerate and decelerate in these regions, based on the information-based objective functional.

\begin{figure}[htbp]
    \centering
    \includegraphics[width=\columnwidth]{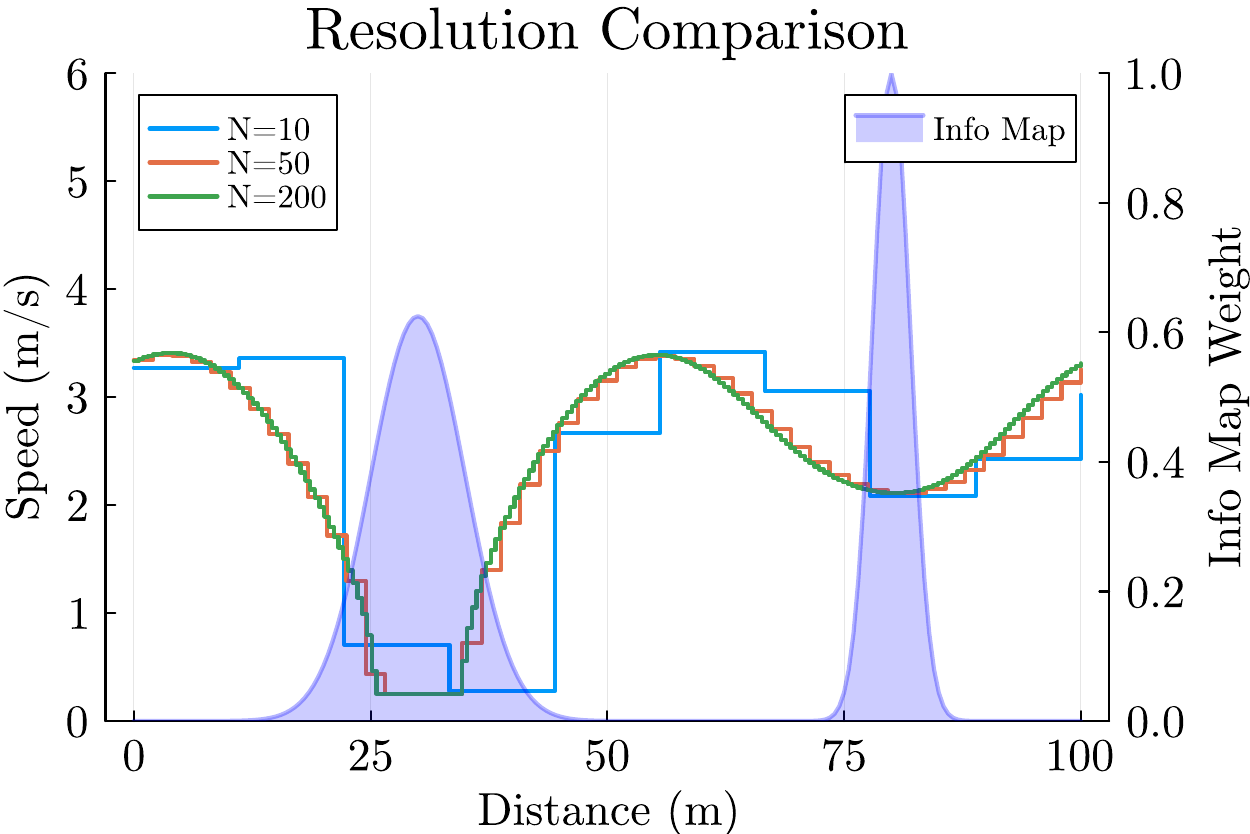}
    \caption{Optimal speed profile along the path for varying discretization levels ($N$). The background heatmap indicates the spatial information distribution, with high-density regions centered at 30 m and 80 m.}
    \label{fig:resolution-comparison}
\end{figure}

\section{Chesapeake Bay Simulation Study}\label{sec:simulation-study}
To demonstrate the applicability of our proposed method to a realistic observational mission, we consider a solar-powered autonomous surface vehicle (ASV) tasked with monitoring salinity along a fixed closed path in the lower Chesapeake Bay. This region exhibits substantial spatial salinity variability due to mixing of freshwater from the bay and saltwater from the ocean, making it a practically relevant monitoring application (see~\cite{stone2019long},~\cite{beegle2013nutrient}).

\subsection{Simulation Setup}
We use salinity data sourced from the NOAA Chesapeake Bay Operational Forecast System~\cite{noaa_cbofs_salinity} to represent the spatiotemporally varying salinity field. We follow an approximately 150 km transect across the mouth of the bay, shown over representative salinity fields in Fig.~\ref{fig:salinity_snapshots}.

The ASV follows the same energy dynamics described in Eq.~\eqref{eq:nv:energy-dynamics}, with an average power input of $P_{\mathrm{in,avg}} = 750\,\mathrm{W}$. The sensor, as described in Eq.~\eqref{eqn:sensing_model}, has a sensing rate $S_0 = 0.005$ and length scale $\sigma=1500$ meters. The path is discretized into $N = 200$ equally sized segments. The ASV completes multiple laps along this path. At the end of each lap, the salinity measurements collected during that traversal are used to update the target-weight profile, and an optimal speed profile for the subsequent lap is obtained by solving Eq.~\eqref{eqn:eoc:ocp}. For comparison, we also simulate a baseline planner that allocates energy uniformly across the path, resulting in a constant speed profile.

We define the weight function based on the salinity measurements, where $\zeta(s)$ denote the salinity (PSU) measured at path location $s$. The weight function is defined from salinity relative to a reference value as:
\begin{equation}
    \gamma(\zeta) = \begin{cases}
        \exp\left(- \frac{(\zeta - \zeta_{\mathrm{ref}})^2}{2\sigma_{\zeta}^2}\right) & \text{if } \zeta \leq \zeta_{\mathrm{ref}} \\
        1 & \text{if } \zeta > \zeta_{\mathrm{ref}}
    \end{cases}
    \label{eq:weight-function}
\end{equation}
where $\sigma_{\zeta}$ is the weight-function spread and $\zeta_{\mathrm{ref}}=30\,\mathrm{PSU}$ is the salinity target. We evaluate $\sigma_{\zeta}\in\{2,3,4\}$ PSU; smaller values produce more concentrated target weights and larger speed variations.

\begin{figure*}
    \centering
    \begin{minipage}[b]{0.30\textwidth}
        \centering
        \includegraphics[width=\textwidth]{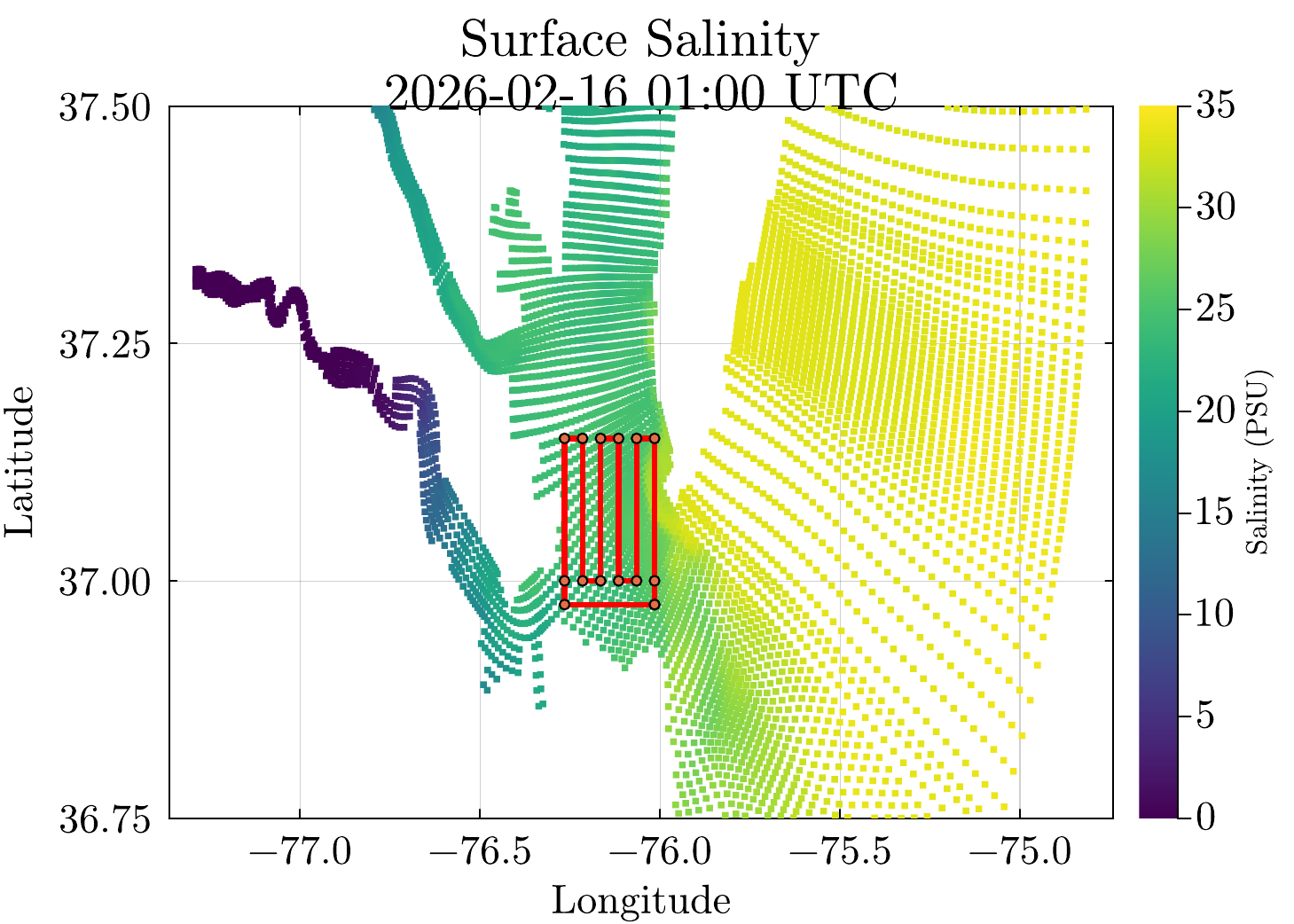}
    \end{minipage}
    \hfill
    \begin{minipage}[b]{0.30\textwidth}
        \centering
        \includegraphics[width=\textwidth]{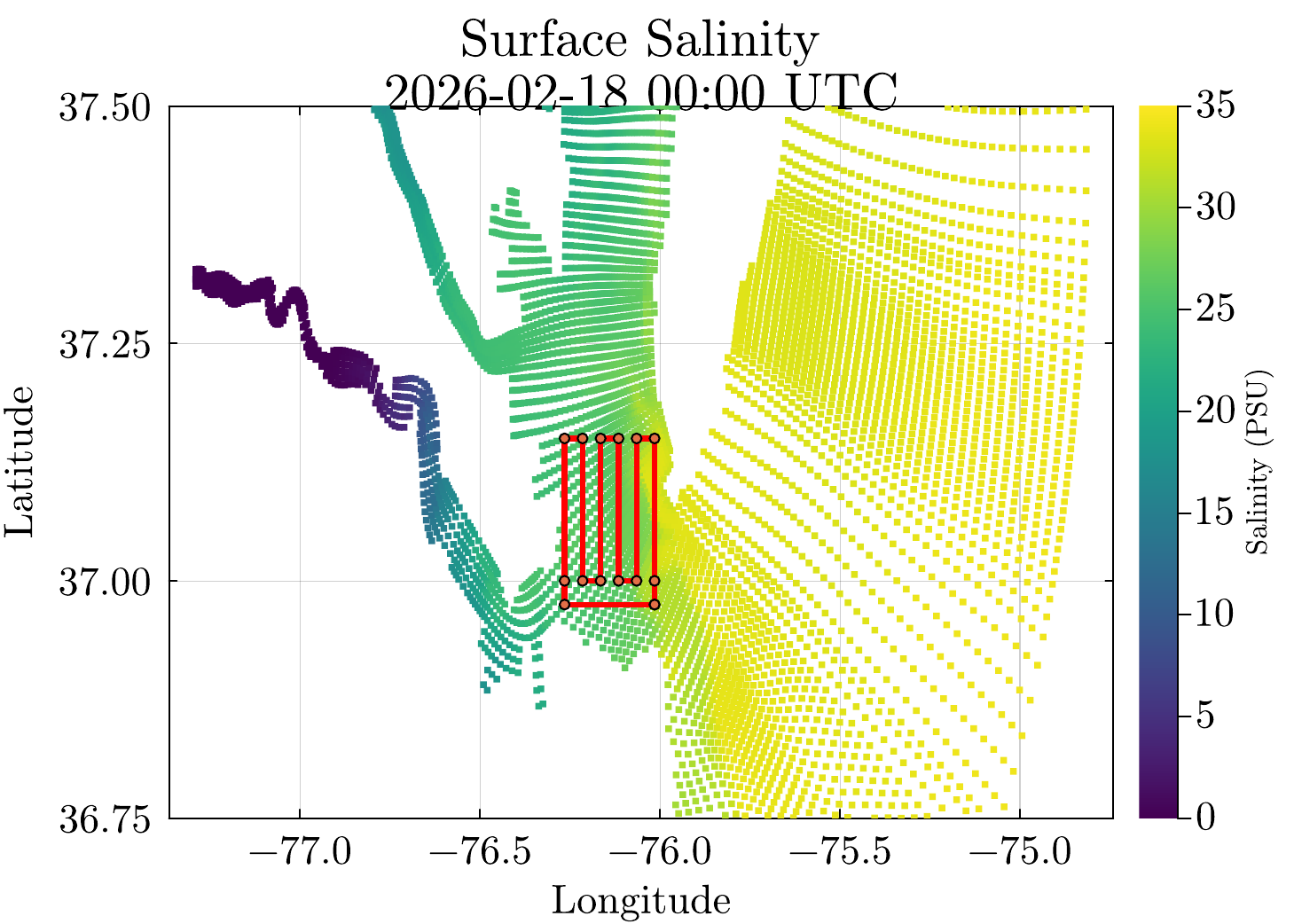}
    \end{minipage}
    \hfill
    \begin{minipage}[b]{0.30\textwidth}
        \centering
        \includegraphics[width=\textwidth]{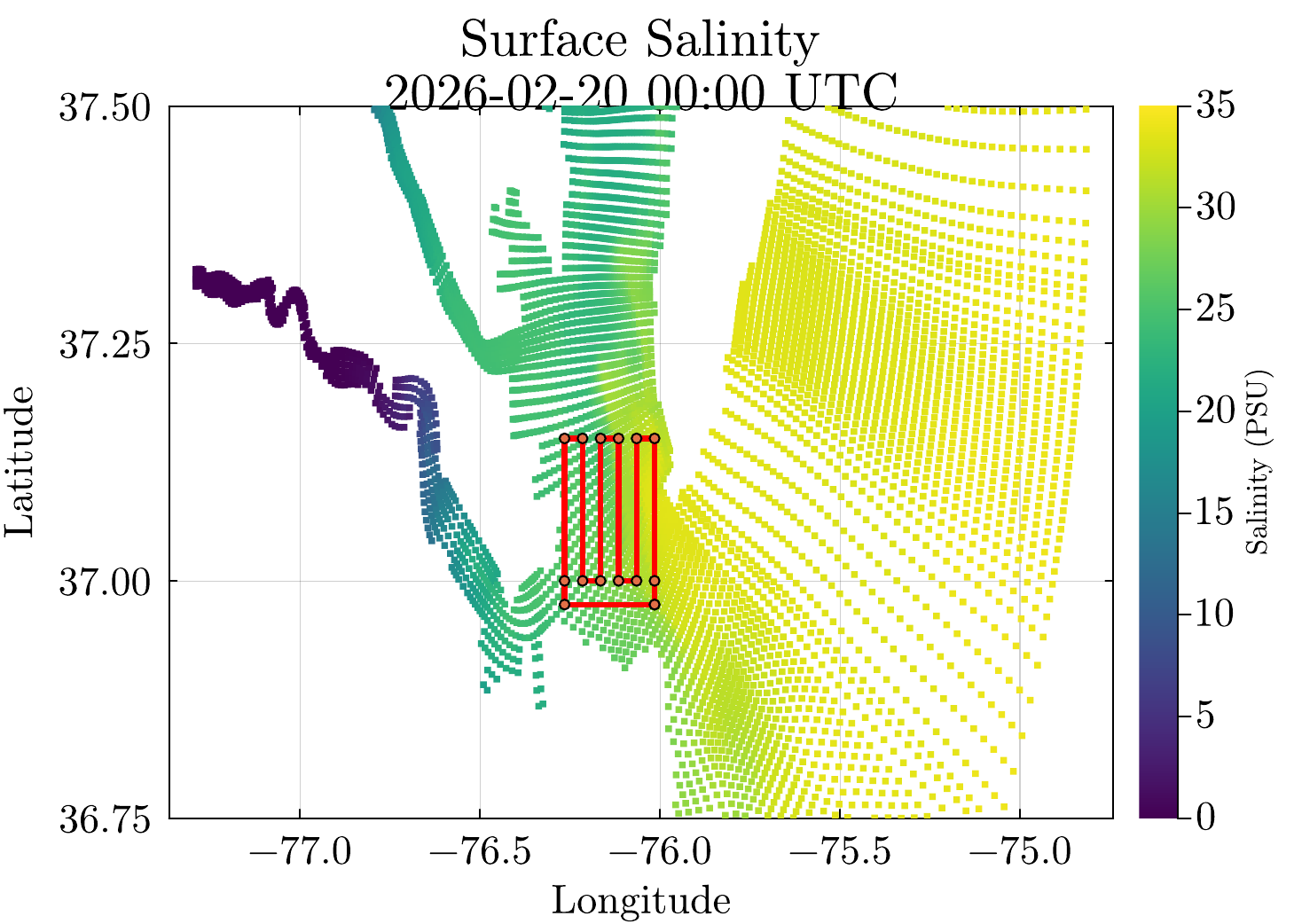}
    \end{minipage}
    \caption{Chesapeake Bay salinity snapshots and the monitored transect.}
    \label{fig:salinity_snapshots}
\end{figure*}

\subsection{Simulation Results}
Table~\ref{tab:simulation-results} compares IVE with a constant-speed baseline under each common weight function. The relative clarity improvement is therefore most appropriately interpreted within a row, since changing $\sigma_{\zeta}$ changes the monitoring objective itself. The most concentrated weighting, $\sigma_{\zeta}=2$, produces the largest improvement in total clarity ($59.5\%$). Increasing the spread to $\sigma_{\zeta}=3$ still provides a substantial $19.6\%$ improvement, and moves the mean salinity measurement closer to the 30 PSU reference ($-1.43$ PSU versus $-2.58$ PSU for the constant-speed strategy). This improvement is accompanied by a larger standard deviation, indicating that the vehicle samples a broader range of salinity values while preferentially allocating time to the target region. The diffuse $\sigma_{\zeta}=4$ case produces a more modest $2.0\%$ gain and measurement statistics closer to those of constant-speed traversal.

The speed profiles in Fig.~\ref{fig:ive_speed_profile} illustrate how the target-weight decay controls the allocation. In each case, the optimizer slows the ASV in high-weight regions and accelerates through locations that are less relevant to the monitoring objective. Smaller values of $\sigma_{\zeta}$ concentrate the reward near the reference salinity and consequently produce stronger variations in speed. In particular, the $\sigma_{\zeta}=2$ and $\sigma_{\zeta}=3$ profiles include near-stationary portions of the trajectory, reflecting the value of spending additional sensing time in localized high-priority regions. As the target weight becomes more diffuse, the allocation becomes smoother: the $\sigma_{\zeta}=4$ profile remains close to the constant-speed baseline while retaining a small clarity advantage. Figure~\ref{fig:ive_speed_time} shows the same behavior over simulation time and highlights that the allocation is recomputed as the observed salinity field changes from lap to lap.

\begin{figure}
    \centering
    \includegraphics[width=\columnwidth]{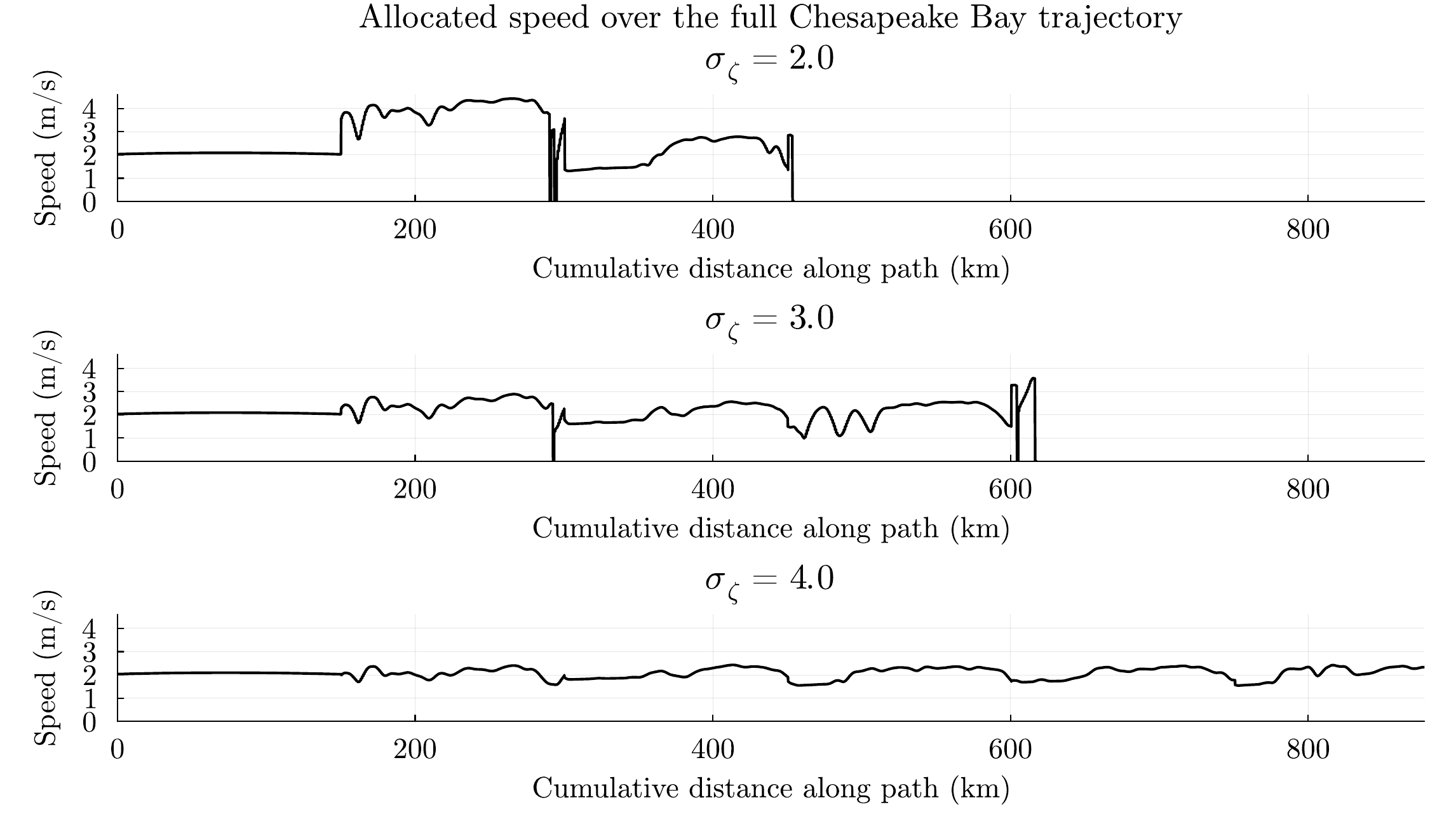}
    \caption{Allocated speed over the complete five-day Chesapeake Bay trajectory, plotted against cumulative distance along the path. Each panel uses the same horizontal scale and includes every lap.}
    \label{fig:ive_speed_profile}
\end{figure}

\begin{figure}
    \centering
    \includegraphics[width=\columnwidth]{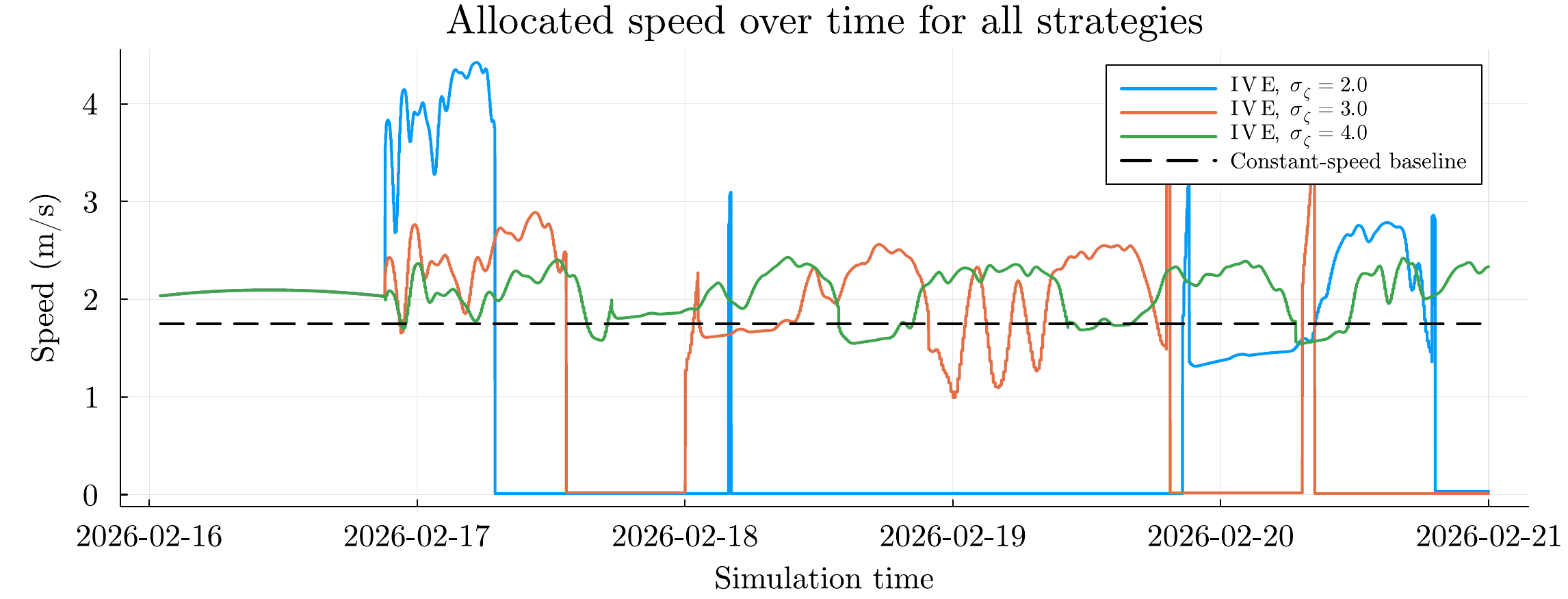}
    \caption{Allocated speed over time for the three IVE target-weight decay settings and the constant-speed baseline.}
    \label{fig:ive_speed_time}
\end{figure}

\begin{table}[htbp]
    \centering
    \caption{Chesapeake Bay simulation results. (IVE/Constant Speed)}
    \label{tab:simulation-results}
    \scriptsize
    \setlength{\tabcolsep}{0pt}
    \begin{tabular}{|c|c|c|c|c|}
        \hline
        $\sigma_{\zeta}$ & \textbf{Clarity (IVE / Constant Speed)} & \textbf{Gain} & \textbf{Mean deviation (PSU)} & \textbf{Std. dev. (PSU)} \\ \hline
        2.0 & 44,528.41 / 27,911.94 & 59.5\% & $-2.54$ / $-2.58$ & 2.10 / 2.34 \\ \hline
        3.0 & 42,761.25 / 35,764.20 & 19.6\% & $-1.43$ / $-2.58$ & 3.03 / 2.34 \\ \hline
        4.0 & 42,234.46 / 41,404.75 & 2.0\% & $-2.38$ / $-2.58$ & 2.65 / 2.34 \\ \hline
    \end{tabular}
\end{table}

\section{Conclusions and Future Work}\label{sec:conclusion}
This paper established a theoretical foundation for persistent sensing in sustainably powered mobile robots. We derived necessary conditions for periodic speed trajectory optimization, establishing analytical results for both spatially varying and uniform weighting functions alongside a high-speed information limit. Numerical optimizations and a Chesapeake Bay case study demonstrated that the optimal strategy decelerates the vehicle in high-importance regions and accelerates through lower-weighted zones. The results show that Information Value of Energy (IVE)-based speed adaptation achieves significantly higher weighted clarity than a uniform-speed baseline under identical power constraints.

Future work will extend this framework to joint path-and-speed co-optimization, receding-horizon control in dynamic environments, and experimental field validation using the SeaTrac SP-48 autonomous surface vehicle (ASV).

\bibliographystyle{IEEEtran}
\bibliography{references.bib}

\end{document}